\pdfoutput=1
\documentclass[11pt,a4paper]{article}
\usepackage[margin=1in]{geometry}
\usepackage{times}
\usepackage{microtype}

\usepackage{amsmath,amsfonts,bm}

\def\eqref#1{equation~\ref{#1}}

\def\1{\bm{1}}

\DeclareMathAlphabet{\mathsfit}{\encodingdefault}{\sfdefault}{m}{sl}
\SetMathAlphabet{\mathsfit}{bold}{\encodingdefault}{\sfdefault}{bx}{n}

\usepackage{hyperref}
\usepackage{url}
\usepackage{graphicx}
\usepackage{booktabs}
\usepackage{amsmath,amssymb,amsthm}
\usepackage{subcaption}
\usepackage{placeins}
\usepackage{flafter}
\usepackage{algorithm}
\usepackage{algpseudocode}
\usepackage{float}
\usepackage[numbers,sort&compress]{natbib}

\hypersetup{
  colorlinks=true,
  linkcolor=black,
  citecolor=black,
  urlcolor=blue,
  pdfborder={0 0 0}
}

\newtheorem{theorem}{Theorem}
\newtheorem{lemma}{Lemma}
\newtheorem{proposition}{Proposition}

\title{Multi-Horizon Consistency as Geometry:\\
When Latent Dynamics Contract, and When They Do Not}

\author{
Kavya Bhand \\
Department of Computer Engineering \\
Vishwakarma Institute of Technology, Pune, India \\
\texttt{kavya.bhand0806@gmail.com}
\and
Aadi Joshi \\
Department of Computer Engineering \\
Vishwakarma Institute of Technology, Pune, India \\
\texttt{toaadijoshi@gmail.com}
}

\begin{document}
\maketitle

\begin{abstract}
Multi-horizon latent consistency is a common training knob in video predictors and world models, but practitioners rarely know what it does to transition geometry.
We treat $\lambda$, the weight on multi-step latent agreement, as a diagnostic control and measure an empirical expansion proxy $L_{20,\mathrm{q95}}$ together with horizon-$20$ prediction error $E_{20}$.
On Moving-MNIST ($n{=}6$ seeds at the critical pair), raising $\lambda$ from $0$ to $0.8$ cuts $L_{20}$ from $4.96{\pm}2.01$ to $1.01{\pm}0.06$ (paired $t$ $p{=}0.005$, Wilcoxon $p{=}0.031$) and halves $E_{20}$ ($0.365{\rightarrow}0.177$, paired $t$ $p{=}1.1{\times}10^{-13}$).
Four of six seeds cross $L{<}1$ at $\lambda{=}0.8$.
The same loss does \emph{not} produce population $L{<}1$ on action-conditioned Pendulum-v1 or CartPole-v1, nor on KTH Actions video, even when $E_{20}$ improves.
An associational mediation analysis on MMNIST gives $\hat{r}{=}0.94$ (95\% CI $[0.88,1.00]$, $n{=}27$, $B{=}2000$); $\lambda$ was not randomized.
Defensive checks (architectural baselines, exogenous stress, WorldTest, MPC, scaling) mostly support a narrow claim: soft consistency can push passive video toward a near-contractive band, and that band is domain-limited.
A stochastic-forcing law $\hat L_{20}{\approx}1.23{+}1.82\eta$ at $\lambda{=}0.8$ (bootstrap slope CI $[1.73,1.92]$, $R^2{=}0.96$) unifies control domains on the same curve via calibrated $\eta_{\mathrm{eff}}$.
Complete joint slices at $\lambda\in\{0.4,1.2\}$ ($30/30$ cells, $5\eta{\times}3$ seeds) show comparable linear $L_{20}(\eta)$ slopes ($\approx 1.69$ and $\approx 2.00$); we do not fit a continuous $(\lambda,\eta)$ surface.
We do not report DreamerV3 or TD-MPC2 returns.
\end{abstract}


\section{Introduction}
\label{sec:intro}

World models and joint-embedding predictors often roll out farther than their training horizon~\citep{ha2018world,hafner2019learning,hafner2023mastering,lecun2022path}.
When they fail, errors compound.
A practical question sits underneath those systems: if you turn up multi-step consistency, what actually changes in the learned transition?

This paper does not propose a new architecture or a leaderboard entry.
It studies one training term,
\[
\mathcal{L}_{\mathrm{cons}}=\sum_{k\in\mathcal{K}}\bigl\|\hat{z}_{t+k}-z^{\mathrm{tf}}_{t+k}\bigr\|_2^2,
\quad \mathcal{K}=\{1,3,5,10,15,20\},
\]
weighted by $\lambda$ against reconstruction and one-step losses, and asks whether $\lambda$ moves a finite-sample expansion proxy $L_{20,\mathrm{q95}}$ across the reference line $L{=}1$.

Practitioners already turn this knob.
What is missing is a domain-aware account of what it does to transition geometry, when population contraction appears, and how that relates to action-conditioned settings where planning actually happens.
We treat the study as a diagnostic: report where the proxy crosses $L{<}1$, where it does not, and whether a single forcing parameter can place those regimes on one curve.

The empirical picture is mixed in a useful way.
Passive Moving-MNIST shows a sharp, seed-replicated drop in both $L_{20}$ and $E_{20}$ as $\lambda$ rises, with several seeds crossing $L{<}1$ near $\lambda{=}0.8$.
The same schedule on Pendulum, CartPole, and KTH tightens spectra or error without population $L{<}1$.
That split is the main result, and it motivates the stochastic-forcing experiment: inject controllable noise $\eta$ into passive transitions and ask whether control domains sit on the same $L_{20}(\eta)$ curve.
They do, approximately linearly, which reframes the passive/active anecdotes as placements on one law rather than isolated failures.
Joint slices at two other $\lambda$ values track the same qualitative rise in $\eta$, so the forcing story is not an artifact of fixing $\lambda{=}0.8$.

\paragraph{Contributions.}
\begin{enumerate}
\item A quantified MMNIST threshold: at $\lambda{=}0.8$, mean $L_{20}$ sits near $1$ ($4/6$ seeds below $1$), with significant paired drops in both $L_{20}$ and $E_{20}$ ($n{=}6$).
\item A negative transfer result for action-conditioned Gymnasium domains and for KTH: $L{<}1$ does not appear at the population level.
\item Associational mediation ($\hat{r}{=}0.94$) with an explicit non-randomization caveat, plus run-level sensitivity ($\hat{r}{\approx}0.80$).
\item A stochastic-forcing law at fixed $\lambda{=}0.8$: $\hat L_{20}\approx a+b\eta$ with bootstrap CIs on $b$, placing Pendulum/CartPole via calibrated $\eta_{\mathrm{eff}}$.
\item A joint $(\lambda,\eta)$ boundary at $\lambda\in\{0.4,1.2\}$ ($30/30$ cells): both slices rise approximately linearly in $\eta$ with slopes near the primary law (descriptive; no formal $\lambda{\times}\eta$ interaction test).
\item Defensive experiments with locked multi-seed numbers: WorldTest ($n{=}5$), MPC ($n{=}3$, real $\lambda{=}0$), and scaling $d{\in}\{16,32,64,128\}$.
\end{enumerate}

\section{Related work}
\label{sec:related}

We organize prior work by theme and state, for each, how the present diagnostic study relates.
The goal is not a leaderboard comparison but a map of \emph{where} a training-time consistency knob touches geometry, planning, and reliability.

\subsection{RSSM, Dreamer, and latent model-based RL}
Recurrent state-space models (RSSMs) learn stochastic latents with reconstruction and rollout losses, then feed imagined trajectories to actors and critics~\citep{hafner2019learning,hafner2021mastering,hafner2023mastering,ha2018world}.
Dreamer-style agents treat the world model as a simulator for policy optimization; TD-MPC2 later couples latent dynamics to gradient-based MPC at scale~\citep{hansen2024tdmpc2}.
Contrastive and DeepMDP objectives also learn latent transitions for control~\citep{kipf2020contrastive,gelada2019deepmdp}.
These systems \emph{implicitly} care about long-horizon behavior (policies are trained on multi-step rollouts) but they rarely report an explicit expansion statistic such as $L_{20,\mathrm{q95}}$.
Our contribution is orthogonal: we hold architecture fixed (GRU residual) and ask what a \emph{standalone} multi-horizon consistency weight $\lambda$ does to a measurable local expansion proxy, without claiming a new agent or benchmark score.

\subsection{JEPA and representation-space prediction}
Joint-embedding predictive architectures predict in representation space rather than pixel space~\citep{lecun2022path,assran2023self,bardes2024revisiting}.
V-JEPA~2 extends the family to action-conditioned video~\citep{assran2025vjepa2}.
JEPAs stabilize training with architectural and objective choices (masking, EMA targets, multi-crop views) that differ from our explicit $\mathcal{L}_{\mathrm{cons}}$ term.
\citet{pendharkar2026exogenous} argue that pure predictive objectives can discard control-relevant stochastic structure, a caution relevant when reading MMNIST contraction as universally desirable.
We include a partial exogenous-digit stress test (reporting $L_{20}$/$E_{20}$ only) to probe whether consistency collapses under injected stochasticity; it does not replicate their retention probes.
Our passive-video results sit in the JEPA-adjacent regime (representation prediction without action labels) but use a simpler GRU transition and a geometric diagnostic rather than a self-supervised leaderboard.

\subsection{Lipschitz bounds, contraction, and stability}
Hard stability constraints (spectral normalization, antisymmetric transitions, stable RNN analyses) bound gain directly~\citep{miyato2018spectral,chang2019antisymmetric,miller2018stable,newhouse2026lipschitz}.
\citet{asadi2018lipschitz} study model-based RL through Wasserstein/Lipschitz model classes: they prove multi-step and value-error bounds that scale with the transition Lipschitz constant and show empirically that an \emph{intermediate} hard clip $k$ on CartPole/Pendulum models maximizes returns when policies are co-trained inside the learned model (see \S\ref{sec:discussion-asadi}).
Our work uses a \emph{soft} multi-horizon agreement loss rather than a per-layer norm cap, and we measure a finite-sample expansion proxy rather than certifying $\|f\|_{\mathrm{Lip}}$.
The textbook horizon-error recurrence (Theorem~\ref{thm:recurrence}) motivates monitoring whether an \emph{estimated} proxy crosses $L{=}1$; Proposition~\ref{prop:proxy} states when $L_{20,\mathrm{q95}}$ is a valid finite-sample statistic for local $k$-step gain (Appendix~\ref{app:proxy}).
Concurrent theory lines pursue contractivity by other means: multi-token prediction gradients~\citep{liu2026lsemtp}, contractive global memory in video stacks~\citep{kairos2026}, Lieb--Robinson locality bounds on pixel dynamics~\citep{diamondlol2026}, and variable-length latent supervision for planning~\citep{vlwm2026}.
Those papers derive or enforce contraction structurally; we empirically \emph{measure} whether a common training knob moves a proxy across a reference line, and document where it fails (action-conditioned control, KTH).

\subsection{Post-hoc and test-time diagnostics (2026)}
A recent wave treats world-model quality as a \emph{diagnostic} rather than only a training loss.
ATM constructs an action-consistency transfer matrix on frozen latent transitions to test whether action semantics survive rollouts~\citep{chen2026atm}.
\citet{ruan2026futurecompatible} score action--state compatibility in World Action Models on robotics simulators and use it for test-time model selection.
AutumnBench WorldTest scores structural understanding via masked-frame prediction and change detection~\citep{warrier2025autumnbench}; we reuse that scoring recipe on our checkpoints but do not claim the full 43-environment suite.
These tools answer ``is this \emph{trained} model reliable for planning or selection?'' at evaluation time.
We instead vary $\lambda$ \emph{during} training and track $L_{20,\mathrm{q95}}$ and $E_{20}$, producing a \emph{training-time} geometry curve.
The axes are complementary: ATM/Ruan-style checks can screen a checkpoint; a $\lambda$ sweep explains how the checkpoint got there.
Neither substitutes for the other.

\subsection{Positioning}
This paper is a mechanistic diagnostic study, not a new world model.
It contributes (i)~a quantified passive-video threshold in $L_{20,\mathrm{q95}}$, (ii)~negative transfer on action-conditioned and natural-video domains, (iii)~associational mediation with explicit non-randomization caveats, (iv)~a stochastic-forcing law linking passive and active regimes at fixed $\lambda$, (v)~complete joint $(\lambda,\eta)$ slices at $\lambda\in\{0.4,1.2\}$, and (vi)~defensive checks (architectural baselines, exogenous stress, WorldTest, MPC, scaling).
It does not compete with Dreamer/TD-MPC2 returns, certify global Lipschitz constants, or replace post-hoc diagnostics.
The open question, why passive video crosses $L{<}1$ while active control does not, is developed in \S\ref{sec:discussion-passive} and tested via $\eta$ in \S\ref{sec:eta-law}.

\paragraph{What this paper is not.}
It is not a Dreamer or JEPA replacement; the architecture is fixed on purpose.
It is not a certified contraction theorem for video world models.
It does not claim that $L{<}1$ improves planning returns---on Pendulum, our MPC result argues the opposite.
It does not identify a causal $\lambda\to L\to E$ effect (mediation is associational).
If you want a new SOTA agent or a universal Lipschitz certificate, this is the wrong paper.
If you want to audit the geometry of a trained latent transition under a consistency weight, that is the target reader.

\section{Preliminaries}
\label{sec:prelim}

\paragraph{Latent transition.}
Let $z_t\in\mathcal{Z}\subseteq\mathbb{R}^d$ denote an encoded latent at time $t$.
A one-step transition map $f:\mathcal{Z}\to\mathcal{Z}$ (optionally conditioned on action $a_t$) induces the $k$-step map $f^{(k)}$.
Teacher-forced latents $z^{\mathrm{tf}}_{t+k}$ come from encoding observed frames or states; free-run predictions $\hat{z}_{t+k}$ roll out $f$ from $z_t$.

\paragraph{Regimes.}
\emph{Passive video} trains $f$ without actions (Moving-MNIST, KTH).
\emph{Action-conditioned control} trains $f(z_t,a_t)$ on offline Gymnasium trajectories (Pendulum-v1, CartPole-v1).
The same consistency weight $\lambda$ is applied in both regimes; the estimand below is regime-agnostic, but the induced geometry need not be.

\paragraph{Estimand.}
For pairs $(z_i,z_j)$ drawn from a validation latent pool,
\[
R^{(k)}_{ij}=\frac{\|f^{(k)}(z_i)-f^{(k)}(z_j)\|}{\|z_i-z_j\|},
\qquad
L_{k,\mathrm{q95}}=\widehat{Q}_{0.95}\bigl(\{R^{(k)}_{ij}\}\bigr).
\]
We report $L_{20,\mathrm{q95}}$ (written $L_{20}$ in tables) and horizon-$20$ prediction MSE $E_{20}$.
$L_{20,\mathrm{q95}}{<}1$ means that on $95\%$ of sampled pairs the $20$-step map is finite-difference non-expansive; it is not a certificate that $\sup\|J_f\|{<}1$ (Appendix~\ref{app:proxy}).

\paragraph{Stochastic forcing.}
To interpolate passive and active geometry we inject
$z\leftarrow f(z)+\eta\,\sigma^\star\varepsilon/\sqrt{d}$, $\varepsilon\sim\mathcal{N}(0,I)$, during training at fixed $\lambda$, and place control domains by calibrated $\eta_{\mathrm{eff}}$ (Section~\ref{sec:eta-law}).

\section{Method}
\label{sec:method}

\subsection{Models}
\textbf{Passive video.} Convolutional encoder, GRU residual transition $\hat{z}_{t+1}=z_t+\alpha\tanh(\Delta z_t)$ with $\alpha{=}0.5$, $d{=}64$, hidden width $160$~\citep{srivastava2015unsupervised}.

\textbf{Action-conditioned control.} MLP encoder of $(s_t,a_t)$; transition consumes $(z_t,a_t)$. Offline random (and, for CartPole, heuristic) Gymnasium trajectories~\citep{tassa2018deepmindcontrol}.

\subsection{Loss and metrics}
Total loss $\mathcal{L}=\mathcal{L}_{\mathrm{recon}}+\mathcal{L}_{\mathrm{1step}}+\lambda\,\mathcal{L}_{\mathrm{cons}}$.
Sweeps use $\lambda\in\{0,0.2,\ldots,1.2\}$ unless noted.

For held-out latent pairs the model rolls out $K{=}20$ steps and forms ratios
$r_{ij}^{(k)}=\|f^{(k)}(z_i)-f^{(k)}(z_j)\|/\|z_i-z_j\|$.
$L_{20,\mathrm{q95}}$ is the $95$th percentile of those ratios.
This is a finite-sample local expansion proxy, not a certified global Lipschitz constant~\citep{miller2018stable,newhouse2026lipschitz}.
$E_{20}$ is MSE on stress-split frames at horizon $20$.

\subsection{Training}
Default: AdamW, lr $2{\times}10^{-3}$, $20$ MMNIST epochs / $25$ Pendulum epochs, seeds $\{42,43,44\}$, with two extra MMNIST seeds at $\lambda\in\{0,0.8\}$ for paired tests ($n{=}6$).
Hyperparameters are listed in Appendix~\ref{app:hyper}.

\subsection{Evaluation protocol}
Unless noted otherwise, each $(\lambda,\mathrm{seed})$ cell is an independent training run from the same data split.
$L_{20,\mathrm{q95}}$ uses $n_{\mathrm{pairs}}{=}240$ random chords from a held-out stress or validation pool with a small number of loader batches (default $4$--$25$ depending on stage; see Appendix~\ref{app:hyper}).
Action-conditioned $L_{20}$ uses matched random action sequences for both members of a pair unless an ablation specifies zero or matched logged actions.
We report mean$\pm$std across seeds and, for the MMNIST critical pair, paired $t$ and Wilcoxon tests.
Bootstrap confidence intervals for mediation and for the $\eta$-law slope use $B{=}2000$ resamples (run-level or seed-block, respectively).
No hyperparameter search over $\lambda$ is claimed beyond the reported grid; $\lambda{=}0.8$ is selected as the MMNIST operating point where mean $L_{20}$ first sits near $1$ with multiple seeds below $1$.

\subsection{Experimental design summary}
The study is organized as a sequence of locked stages rather than an open-ended leaderboard chase:
\begin{enumerate}
\item \emph{Primary geometry curve} on Moving-MNIST ($\lambda$ grid; critical pair $n{=}6$).
\item \emph{Transfer negatives} on Pendulum, CartPole, and KTH under the same $\lambda$ grid.
\item \emph{Mechanism} via stochastic forcing $\eta$ at fixed $\lambda{=}0.8$ (primary $n{=}50$).
\item \emph{Defenses}: architectural baselines, exogenous digit, WorldTest-style scoring, latent MPC, latent-dimension scaling.
\item \emph{Joint} $(\lambda,\eta)$ boundary at $\lambda\in\{0.4,1.2\}$ ($30/30$ cells; Appendix~\ref{app:eta-joint}).
\end{enumerate}
Every claim in the abstract is backed by a locked CSV under \texttt{results/master/}
and checked by the package validator (Appendix~\ref{app:repro}).

\begin{algorithm}[H]
\caption{Train and evaluate a consistency-weighted latent world model}
\label{alg:pipeline}
\begin{algorithmic}[1]
\Require Dataset $\mathcal{D}$, weight $\lambda$, horizons $\mathcal{K}$, noise level $\eta$ (default $0$)
\State Initialize encoder, transition $f$, decoder
\For{epoch $=1,\ldots,E$}
  \For{minibatch $B\subset\mathcal{D}$}
    \State Encode teacher-forced latents $\{z_t\}$; optionally add noise of scale $\eta\sigma^\star$
    \State Compute $\mathcal{L}_{\mathrm{recon}}+\mathcal{L}_{\mathrm{1step}}+\lambda\sum_{k\in\mathcal{K}}\|\hat{z}_{t+k}-z^{\mathrm{tf}}_{t+k}\|_2^2$
    \State AdamW step with gradient clipping
  \EndFor
\EndFor
\State Encode validation pool; sample $n_{\mathrm{pairs}}$ chords; compute $L_{20,\mathrm{q95}}$ and $E_{20}$
\State \Return checkpoint and metrics
\end{algorithmic}
\end{algorithm}

\FloatBarrier

\section{Theoretical background}
\label{sec:theory}

We use two theoretical pieces: a horizon-error recurrence under Lipschitz dynamics (standard, but proved in full in Appendix~\ref{app:recurrence}), and a validity argument for the finite-sample statistic $L_{20,\mathrm{q95}}$ (Appendix~\ref{app:proxy}).

\subsection{Recurrence bound}
Let $f:\mathcal{Z}\to\mathcal{Z}$ be a latent transition map and $\hat{z}_{t+1}=f(\hat{z}_t)+\xi_t$ a learned rollout with per-step disturbance $\|\xi_t\|\le\epsilon$.
Write $e_t=\|\hat{z}_t-z_t\|$ for rollout error against a reference trajectory.
\begin{theorem}[Horizon error recurrence]
\label{thm:recurrence}
If $f$ is $L$-Lipschitz on the rollout manifold and $e_0=0$, then for horizon $h\ge 1$,
\begin{equation}
e_h \le L^{h} e_0 + \epsilon\,\frac{1-L^{h}}{1-L}
\qquad (L\neq 1),
\qquad
e_h \le h\epsilon \qquad (L=1).
\end{equation}
In particular, if $L<1$ the disturbance term saturates as $h\to\infty$; if $L>1$, error can grow exponentially.
\end{theorem}
The bound is classical in spirit~\citep{asadi2018lipschitz,bacciotti2006lyapunov,talvitie2014}; Appendix~\ref{app:recurrence} gives a step-by-step proof by induction.
We use it only to motivate monitoring an \emph{estimated} expansion statistic: crossing a proxy below $1$ is suggestive of non-explosive $k$-step gain on the measured set, not a certificate that $\sup\|J_f\|<1$.

\subsection{What \texorpdfstring{$L_{20,\mathrm{q95}}$}{L20,q95} estimates}
For held-out latents $z_i$ and $k$-step map $f^{(k)}$, define pair ratios
$r_{ij}^{(k)}=\|f^{(k)}(z_i)-f^{(k)}(z_j)\|/\|z_i-z_j\|$.
Our metric $L_{20,\mathrm{q95}}$ is the $95$th percentile of $\{r_{ij}^{(20)}\}$ over $n_{\mathrm{pairs}}$ random pairs from validation latents (implementation: \texttt{src/metrics.py}).
\begin{proposition}[Proxy validity (informal)]
\label{prop:proxy}
Suppose $f$ is $L$-Lipschitz on the support of the validation latent distribution.
Then every sampled ratio satisfies $r_{ij}^{(k)}\le L^{k}$.
If additionally $f$ is $C^1$, pairs are drawn with $\|z_i-z_j\|\le\delta$ for small $\delta$, and rollout actions are fixed across the pair (passive video or matched-action protocol), then
$r_{ij}^{(k)}$ concentrates around the directional gain $\|(J_f(z_i)^{k}(z_j-z_i))\|/\|z_j-z_i\|$ as $\delta\to 0$.
The empirical $95$th percentile is a consistent estimator of the $95$th quantile of $r_{ij}^{(k)}$ under i.i.d.\ pair resampling as $n_{\mathrm{pairs}}\to\infty$.
\end{proposition}
Appendix~\ref{app:proxy} states coverage conditions, relates $k$-step ratios to composed Jacobians, and connects the statistic to a one-step spectral-radius probe.
A short Jacobian power-iteration check on MMNIST (Appendix~\ref{app:jacobian}) shows both $L_{20,\mathrm{q95}}$ and mean one-step spectral radius drop when $\lambda$ rises ($1.29\to1.07$ and $1.29\to1.02$ at six epochs), consistent with Proposition~\ref{prop:proxy} but not certifying a global Lipschitz constant.

\section{Results}
\label{sec:results}

\subsection{Moving-MNIST: a rapid \texorpdfstring{$\lambda$}{lambda} threshold}
Table~\ref{tab:mmnist} and Figure~\ref{fig:phase_mmnist} summarize $27$ MMNIST runs, with $n{=}6$ at $\lambda\in\{0,0.8\}$.
$L_{20}$ falls from $4.96{\pm}2.01$ to $1.01{\pm}0.06$ (paired $t$ $p{=}0.005$; Wilcoxon $p{=}0.031$).
$E_{20}$ falls from $0.365{\pm}0.0007$ to $0.177{\pm}0.0003$ (paired $t$ $p{=}1.1{\times}10^{-13}$; Wilcoxon $p{=}0.031$).
Seeds $43,45,46,47$ sit below $L{=}1$ at $\lambda{=}0.8$ ($4/6$).
Mean $L_{20}$ at that setting is on the boundary rather than deep in a contractive regime.

Call this an \emph{empirical regime change}: a rapid threshold in measured $L_{20}$ and $E_{20}$ as $\lambda$ increases.
No non-analytic statistical-mechanics phase transition is claimed.

\begin{table}[ht]
\centering
\caption{Moving-MNIST $\lambda$-sweep (mean $\pm$ std). $n{=}6$ at $\lambda\in\{0,0.8\}$; otherwise $n{=}3$.}
\label{tab:mmnist}
\small
\begin{tabular}{lccc}
\toprule
$\lambda$ & $L_{20,\mathrm{q95}}$ & $E_{20}$ & frac.\ $L{<}1$ \\
\midrule
0.0 & $4.96 \pm 2.01$ & $0.365 \pm 0.0007$ & 0/6 \\
0.2 & $1.04 \pm 0.03$ & $0.177 \pm 0.0003$ & 0/3 \\
0.4 & $1.06 \pm 0.03$ & $0.177 \pm 0.0001$ & 0/3 \\
0.6 & $1.04 \pm 0.04$ & $0.177 \pm 0.0002$ & 0/3 \\
0.8 & $1.01 \pm 0.06$ & $0.177 \pm 0.0003$ & 4/6 \\
1.0 & $1.06 \pm 0.06$ & $0.177 \pm 0.0003$ & 0/3 \\
1.2 & $1.12 \pm 0.04$ & $0.177 \pm 0.0004$ & 0/3 \\
\bottomrule
\end{tabular}
\end{table}

\begin{figure}[ht]
\centering
\includegraphics[width=0.48\linewidth]{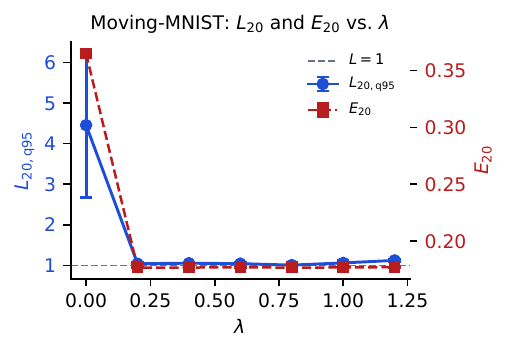}
\includegraphics[width=0.48\linewidth]{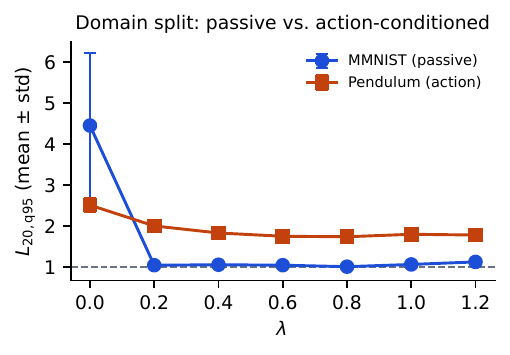}
\caption{Left: MMNIST $L_{20}$ and $E_{20}$ vs.\ $\lambda$. Right: passive MMNIST vs.\ action-conditioned Pendulum (mean $\pm$ std; dashed line $L{=}1$).}
\label{fig:phase_mmnist}
\end{figure}
\FloatBarrier

\subsection{Pendulum and CartPole: tightening without \texorpdfstring{$L{<}1$}{L<1}}
On Pendulum-v1, $L_{20}$ decreases from $2.51{\pm}0.17$ at $\lambda{=}0$ ($n{=}5$) to $1.74{\pm}0.13$ at $\lambda{=}0.8$ ($n{=}3$), then stays near $1.78$ at $\lambda{=}1.2$.
No seed crosses $L{<}1$.
$E_{20}$ is already near $10^{-4}$, so error is a weak phase detector.
Action-protocol ablations at $\lambda{=}0.8$ (random / zero / matched actions) leave $L_{20}$ in $\{1.73,1.81,1.85\}$ (Appendix~\ref{app:pendulum_mech}).

CartPole-v1 shows the same qualitative split:
$L_{20}$ stays expansive at every $\lambda$ with $n{=}5$ ($2.96{\pm}0.26$, $3.18{\pm}0.24$, $3.17{\pm}0.42$, $2.99{\pm}0.20$ at $\lambda\in\{0,0.4,0.8,1.2\}$) while $E_{20}$ drops once $\lambda{\ge}0.4$ (Appendix~\ref{app:cartpole}).
On dm\_control cartpole-swingup (offline random trajectories, same GRU residual backbone), a CPU sweep yields $\lambda{=}0$: $L_{20}{=}1.41{\pm}0.01$ ($n{=}3$); $\lambda{=}0.4$: $1.50{\pm}0.01$ ($n{=}3$); $\lambda{=}0.8$: $1.55{\pm}0.13$ ($n{=}3$); $\lambda{=}1.2$: $1.55{\pm}0.06$ ($n{=}2$; seed~$44$ not completed).
No population $L{<}1$ appears on this grid, matching Gymnasium CartPole-v1 qualitatively (tighter absolute $L_{20}$, still expansive).
We treat this as a recognized-control sanity check rather than a Dreamer-scale benchmark.

\subsection{KTH Actions}
On real KTH clips ($n{=}3$, seq\_len $16$), $L_{20}$ tightens from $3.32{\pm}2.57$ at $\lambda{=}0$ to $1.54{\pm}0.66$ at $\lambda{=}1.2$.
Every seed decreases, seed $44$ remains an expansive outlier, and no seed crosses $L{<}1$.
$E_{20}$ stays near $0.05$.
A noisy three-digit MNIST proxy (Appendix~\ref{app:s3}) is retained only as a stress test; KTH is the natural-video result.

\subsection{Latent dimension}
For $d\in\{16,32,128\}$ we sweep $\lambda\in\{0,0.6,0.8\}$ with $n{=}3$; $d{=}64$ is the main MMNIST table.
At $d{=}16$, mean $L_{20}{=}0.91$ already at $\lambda{=}0$ ($2/3$ seeds $L{<}1$).
At $d{=}32$, $L_{20}$ falls $3.47{\rightarrow}0.89$ by $\lambda{=}0.6$ ($1.01$ at $\lambda{=}0.8$).
At $d{=}128$, $L_{20}$ falls $4.96{\rightarrow}1.14$ and never crosses $L{<}1$.
Larger latents need stronger $\lambda$ to approach the contractive band.
This finite-size pattern is consistent with an under-constrained residual map in high $d$: soft multi-horizon agreement must fight more directions before the $95$th percentile of chord ratios drops below $1$.

\subsection{Stochastic forcing: a linear expansion law}
\label{sec:eta-law}

To unify the passive/active split, we inject action-like noise into passive MMNIST transitions during training:
$z\leftarrow f(z)+\eta\,\sigma^\star\varepsilon/\sqrt{d}$ with $\varepsilon\sim\mathcal{N}(0,I)$, fixed $\lambda{=}0.8$, and $\sigma^\star{=}0.5$.
Primary GPU sweep: $\eta\in\{0,0.1,\ldots,1.2\}$ (10 points) $\times$ 5 seeds ($n{=}50$ runs, 16 epochs).
Mean $L_{20,\mathrm{q95}}$ increases approximately linearly,
$\hat L_{20}\approx 1.23 + 1.82\eta$
(block-bootstrap 95\% CI on slope $[1.73,1.92]$, $R^2{=}0.96$; Figure~\ref{fig:eta}).
\begin{figure}[H]
\centering
\includegraphics[width=0.78\linewidth]{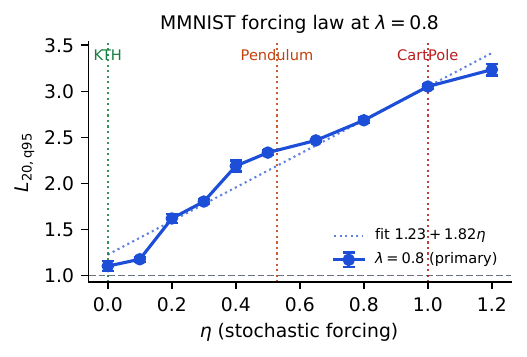}
\caption{Stochastic forcing at $\lambda{=}0.8$: mean $L_{20,\mathrm{q95}}$ vs.\ $\eta$ on MMNIST ($n{=}5$ seeds per point; SEM bars). Dotted line: least-squares fit. Vertical dotted lines: calibrated $\eta_{\mathrm{eff}}$ for Pendulum, CartPole, and KTH.}
\label{fig:eta}
\end{figure}
We do \emph{not} observe a mean-level $L{=}1$ upcrossing on this grid ($\hat L_{20}(\eta{=}0){=}1.10$, per-$\eta$ 95\% CI $[0.95,1.26]$), but calibrated effective noise levels place Pendulum ($\eta_{\mathrm{eff}}{\approx}0.53$) and CartPole ($\eta_{\mathrm{eff}}{\approx}1.0$) at predicted $\hat L_{20}\approx 2.2$ and $3.1$, consistent with population $L{>}1$ in those domains.
KTH is the exception under the same calibration: $\eta_{\mathrm{eff}}{\approx}0$ by construction (no actions), yet $L_{20}$ stays expansive; we treat that as a natural-video residual outside the forcing law rather than a contradiction of the Pendulum/CartPole placements.
Primary bootstrap CIs ($n{=}5$ seeds per $\eta$) are reported in \texttt{eta\_law\_bootstrap.json}.
Joint boundary slices at $\lambda\in\{0.4,1.2\}$ are complete ($30/30$ cells; Appendix~\ref{app:eta-joint}) and show comparable linear $L_{20}(\eta)$ slopes; we do not claim a denser continuous $(\lambda,\eta)$ surface or a finer $\eta$ grid on the primary curve.

\section{Associational mediation}
\label{sec:mediation}

We fit Baron--Kenny paths on z-scored $\lambda$, $\log L_{20}$, and $\log E_{20}$ with DAG $\lambda\to L\to E$ and a direct $\lambda\to E$ edge~\citep{baron1986moderator,pearl2009causal}.
\begin{figure}[H]
\centering
\includegraphics[width=0.55\linewidth]{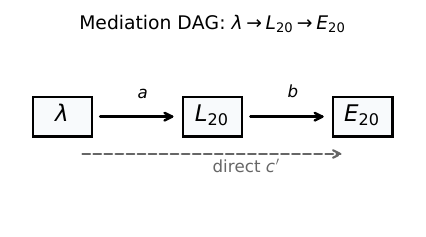}
\caption{Associational mediation DAG used for the path-ratio analysis. Solid arrows: $\lambda\to L_{20}\to E_{20}$; dashed: direct $\lambda\to E_{20}$.}
\label{fig:dag}
\end{figure}
Point estimates use seven $\lambda$-averaged means.
Confidence intervals use run-resampled bootstrap ($B{=}2000$): resample $27$ runs, re-aggregate, recompute paths.

Primary result: $\hat{r}{=}0.94$, 95\% CI $[0.88,1.00]$.
Run-level sensitivity on all $27$ tuples: $\hat{r}{\approx}0.80$, CI $[0.65,0.95]$.
Because $\lambda$ is a training hyperparameter rather than a randomized intervention, these numbers are associational.
They do not identify a causal effect in the Pearl sense.

Interpretation is deliberately cautious.
A large path ratio is consistent with the story that $\lambda$ associates with geometry and that geometry associates with long-horizon error on MMNIST.
It is not a license to say ``raise $\lambda$ to contract $L$, then $E$ falls.''
Randomizing $\lambda$ (or $\eta$) inside a fixed architecture would be the next design step for causal language (Section~\ref{sec:open}).

\section{Defensive experiments}
\label{sec:defensive}

\subsection{What we do not compare}
DreamerV3 DMC returns, TD-MPC2 multi-task returns, RSSM stochastic-latent variants, and a full AutumnBench leaderboard are out of scope~\citep{hafner2023mastering,hansen2024tdmpc2,warrier2025autumnbench}.
Those omissions are intentional. The paper is a diagnostic study of $\lambda$.
Comparing against agent returns would change the estimand from geometry to control performance and would require a different experimental budget.

\subsection{Architectural baselines}
At $\lambda{=}0.8$ on MMNIST ($n{=}3$), soft consistency ($L_{20}{=}1.04{\pm}0.08$) matches spectral normalization ($1.05{\pm}0.16$) and beats AntisymmetricRNN ($5.22{\pm}0.38$), which over-constrains the dynamics (Figure~\ref{fig:baselines}, left).
Soft consistency is therefore competitive with a hard spectral constraint on this passive-video setting, without requiring a custom transition parameterization.
We do not claim superiority outside MMNIST; the point is that $\lambda$ is not a strictly weaker knob than spectral normalization for the $L_{20}$ proxy here.

\subsection{Exogenous digit}
Following the spirit of~\citet{pendharkar2026exogenous}, a digit with resampled random velocity is added and only $L_{20}$/$E_{20}$ are tracked.
At $\lambda{=}1.2$, exogenous $L_{20}{=}0.97{\pm}0.05$ and control $L_{20}{=}1.17{\pm}0.20$ ($n{=}3$).
No MSE/$L$ collapse appears.
That is not a refutation of exogenous discard in JEPAs; feature-retention probes are absent.

\subsection{WorldTest-style scoring}
In-domain MMNIST behavioral scores ($n{=}5$ both arms): $0.41{\pm}0.08$ at $\lambda{=}0$ to $1.06{\pm}0.09$ at $\lambda{=}0.8$ (paired $\Delta{+}0.64$).
Masked-frame prediction stays near chance; change detection drives the gain.
On a grid-world pixel adapter (OOD), scores move $1.88{\pm}0.35\rightarrow 1.17{\pm}0.52$ ($\Delta{-}0.70$): consistency helps in-domain CD and hurts the OOD adapter.
AutumnBench leaderboard evidence is not claimed.

\subsection{Latent MPC (negative sweet-spot)}
Random-shooting MPC on Pendulum ($H{=}10$, $K{=}32$, $20$ episodes, seeds $\{42,43,44\}$) with real $\lambda{=}0$ checkpoints:
mean returns $-713{\pm}62$ ($\lambda{=}0$), $-756{\pm}12$ ($\lambda{=}0.8$), $-768{\pm}23$ ($\lambda{=}1.2$).
Best mean at $\lambda{=}0$; $\lambda{=}0.8$ wins on $1/3$ seeds.
There is no planning sweet spot at the MMNIST $\lambda$ threshold.
That ranking differs from an earlier single-seed run that used a $\lambda{=}0.2$ proxy for the expansive regime; the multi-seed result with true $\lambda{=}0$ supersedes it.
See \S\ref{sec:discussion-asadi} for reconciliation with~\citet{asadi2018lipschitz}.

The negative MPC result is load-bearing for the paper's framing.
If $L{<}1$ on passive video were a sufficient condition for better latent planning on control, we would expect $\lambda{=}0.8$ to help on Pendulum.
It does not under this planner.
That keeps the diagnostic honest: geometry and planning utility are separable claims.

\begin{figure}[ht]
\centering
\includegraphics[width=0.45\linewidth]{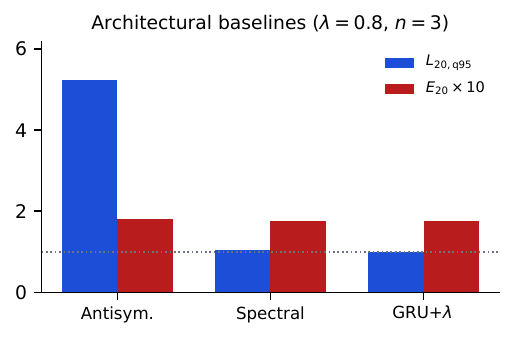}
\includegraphics[width=0.45\linewidth]{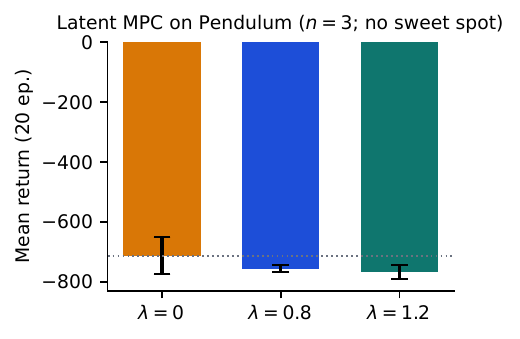}
\caption{Left: architectural baselines at $\lambda{=}0.8$. Right: Pendulum latent-MPC returns ($n{=}3$ seeds, $20$ episodes each).}
\label{fig:baselines}
\end{figure}
\FloatBarrier

\section{Discussion}
\label{sec:discussion}

The usable claim is narrow: soft consistency contracts passive-video geometry under the $L_{20}$ proxy, and that effect does not automatically transfer to action-conditioned or natural-video settings.

\subsection{Passive versus active contraction}
\label{sec:discussion-passive}

Moving-MNIST is \emph{passive}: transitions are learned from unlabeled video, actions are absent at train time, and the consistency term aligns teacher-forced multi-step latents on a low-dimensional digit manifold.
Raising $\lambda$ sharply moves $L_{20,\mathrm{q95}}$ toward $1$ and cuts $E_{20}$, four of six seeds cross $L{<}1$ at $\lambda{=}0.8$.
Pendulum, CartPole, and KTH are \emph{active} or \emph{natural-video} regimes: action channels (or richer appearance statistics) enter the transition, offline data are simulator rollouts or human motion clips, and the same $\lambda$ tightens spectra or error without population $L{<}1$.

We read this as a domain boundary, not a failed hyperparameter sweep.
Passive video offers a nearly deterministic latent flow on which multi-horizon agreement can act like implicit spectral regularization, pushing local $k$-step gain toward a non-expansive band (Proposition~\ref{prop:proxy}).
Once actions (or exogenous visual factors) inject directions that must be preserved for control or realism, the same loss still improves one-step or short-horizon fit in places but no longer collapses the $95$th percentile of $k$-step ratios below $1$.
\citet{pendharkar2026exogenous} warn that aggressive predictive objectives can discard stochastic degrees of freedom needed downstream; our action-conditioned negatives are consistent with that picture, though we do not run their retention probes.

Why might passive video cross and active control not?
Three mechanisms are plausible and not mutually exclusive:
\begin{enumerate}
\item \emph{Identifiability.} Passive digits lie on a low-curvature manifold where contraction and prediction align; control latents must encode torque-relevant state, enlarging local Jacobian variation.
\item \emph{Loss competition.} $\mathcal{L}_{\mathrm{cons}}$ can dominate passive training at high $\lambda$; on Pendulum/CartPole, near-zero $E_{20}$ already decouples error from $L_{20}$, so geometry and MSE need not move together.
\item \emph{Action routing.} Matched-action ablations leave $L_{20}\approx 1.7$--$1.9$ at $\lambda{=}0.8$ (Appendix~\ref{app:pendulum_mech}), suggesting action-conditioned Jacobians resist the same isotropic shrinkage that passive video admits.
\end{enumerate}
Section~\ref{sec:eta-law} tests (i)--(iii) via controllable $\eta$: a single linear law with bootstrap confidence bands places passive video near $\eta{=}0$ and control domains at higher $\eta_{\mathrm{eff}}$, reducing four disconnected anecdotes to one reported scaling relation.

\subsection{Reconciling planning sweet spots with \citet{asadi2018lipschitz}}
\label{sec:discussion-asadi}

\citet{asadi2018lipschitz} bound multi-step Wasserstein error and value error in terms of a model Lipschitz constant, then empirically cap weight norms by a hard bound $k$ while learning CartPole ($15{\times}10^{3}$ tuples) and Pendulum ($10^{4}$ tuples) transition models.
For each $k$, they select the median cross-validation model and train actor-critic / deterministic policy gradient inside the learned model, evaluating returns on the true environment.
Figure~6 of their paper shows an intermediate $k$ maximizing average return, a planning sweet spot on spectral gain during coupled model--policy learning.

Our Pendulum latent-MPC experiment is deliberately decoupled: $\lambda$ is varied only at train time, checkpoints are frozen, and a random-shooting planner searches action sequences in latent space ($H{=}10$, $K{=}32$, $n{=}3$ seeds).
Mean returns are best at $\lambda{=}0$ ($-713{\pm}62$) and worsen at $\lambda{=}0.8$ and $1.2$ ($-756{\pm}12$, $-768{\pm}23$); there is no sweet spot at the MMNIST contraction threshold.

These results are compatible rather than contradictory:
\begin{itemize}
\item Different knobs: Asadi control $k$, a hard per-layer Lipschitz cap on the transition network; we control $\lambda$, a soft penalty on multi-horizon latent agreement. Either can shrink local gain without tracking the other.
\item Different planning loops: their sweet spot arises when policy gradients exploit a model whose gain was tuned jointly with control; our negative result shows that a passive-video-style contraction point need not help a fixed latent planner on a control checkpoint.
\item Different readouts: they optimize environment return; we report an expansion proxy $L_{20,\mathrm{q95}}$ and a crude MPC score. Improved geometry on MMNIST does not license transferring $\lambda{=}0.8$ to Pendulum planning.
\end{itemize}
Takeaway for practitioners: spectral regularization can help planning when the regularizer, model, and planner are co-designed; a training-time $\lambda$ that crosses $L{<}1$ on passive video is neither necessary nor sufficient for latent MPC gains on the control domains tested here.

\subsection{Diagnostics, mediation, and open questions}

ATM and future-compatibility diagnostics~\citep{chen2026atm,ruan2026futurecompatible} screen \emph{trained} models for action-consistent rollouts; concurrent theory work enforces or analyzes contractivity via multi-token losses, memory gating, or locality bounds~\citep{liu2026lsemtp,kairos2026,diamondlol2026,vlwm2026}.
Our $\lambda$ sweep is a \emph{training-time} geometry curve on a fixed GRU backbone.
Use all three axes when auditing a world model: post-hoc action/future tests, structural contractivity theory, and a consistency ablation with $L_{20,\mathrm{q95}}$.

Associational mediation ($\hat{r}{=}0.94$ on MMNIST) supports a coherent story: $\lambda$ associates with lower $L_{20}$ and lower $E_{20}$, but $\lambda$ was not randomized, so causal language is avoided.
The passive/active split is partly unified by the stochastic-forcing law (\S\ref{sec:eta-law}): soft consistency is a domain-limited geometry diagnostic---useful on passive video at low $\eta_{\mathrm{eff}}$, and misleading if exported wholesale to control without remeasurement on the same curve.
Joint slices at $\lambda\in\{0.4,1.2\}$ (Appendix~\ref{app:eta-joint}) track the same qualitative rise in $\eta$, which argues against treating the primary $\lambda{=}0.8$ slope as a one-off.

\section{Limitations}
\label{sec:limitations}

Sample sizes are modest: $n{=}6$ at the critical MMNIST pair, $n{=}5$ on CartPole (all $\lambda$) and on Pendulum at $\lambda{=}0$, $n{=}4$ on Pendulum at $\lambda{=}0.4$, $n{=}3$ elsewhere unless noted, and $n{=}5$ seeds on the primary $\eta$ grid.
For the MMNIST paired $L_{20}$ drop ($4.96\to1.01$), the observed effect is large relative to within-seed variability (paired $t$ $p{=}0.005$), so the critical-pair comparison is not underpowered for detecting a large shift; smaller effects elsewhere would need more seeds.
$L_{20,\mathrm{q95}}$ is a local proxy whose validity conditions are in Appendix~\ref{app:proxy}; it is not $\sup\|J_f\|$.
Mediation is associational.
Benchmarks are Moving-MNIST, Pendulum-v1, CartPole-v1, KTH, and dm\_control cartpole-swingup (CPU sweep: $11/12$ cells locked; $\lambda{=}1.2$ missing seed~$44$).
WorldTest uses an AutumnBench-style scorer on our checkpoints, not the 43-environment suite.
MPC is random shooting only.
The exogenous test lacks retention probes.
Architecture is GRU-only.
The joint $(\lambda,\eta)$ boundary is complete at $\lambda\in\{0.4,1.2\}$ ($n{=}3$ per cell; Appendix~\ref{app:eta-joint}): both slices rise approximately linearly in $\eta$, with slopes comparable to the primary $\lambda{=}0.8$ law; we still do not fit a full continuous surface over $\lambda$.
Logged compute is on the order of $50$ GPU-hours plus multi-day CPU fill-ins (Appendix compute table).

\subsection{Threats to validity}
\paragraph{Internal.}
$\lambda$ is a training hyperparameter, not a randomized intervention, so $\lambda\to L\to E$ paths are associational.
$L_{20,\mathrm{q95}}$ depends on pair sampling and validation pools; Appendix~\ref{app:proxy} lists coverage conditions.
\paragraph{External.}
Passive-video contraction need not transfer to richer video (KTH already fails at the population level) or to Dreamer-scale agents with larger latents and coupled planners.
\paragraph{Construct.}
$L_{20}{<}1$ is a finite-difference non-expansiveness event, not a planning-quality certificate; our MPC negative result makes that distinction operational.

\section{Conclusion}
\label{sec:conclusion}

Raising multi-horizon consistency on Moving-MNIST produces a clear, seed-replicated drop in an expansion proxy and in long-horizon error.
The same knob does not deliver population contraction on the action-conditioned and natural-video settings we tested.
A controllable stochastic-forcing sweep places those regimes on one approximately linear $L_{20}(\eta)$ curve at fixed $\lambda{=}0.8$, and complete joint slices at $\lambda\in\{0.4,1.2\}$ show the same qualitative rise.
That is enough to treat $\lambda$ as a geometry diagnostic with known domain limits.
It is not a universal contraction law, and it is not a new world model.

\paragraph{Practical takeaway.}
For passive-video latents, log $L_{20,\mathrm{q95}}$ with MSE and treat $\lambda$ as a geometry dial with a measurable threshold near $1$.
For action-conditioned control, remeasure $L_{20}$ on that domain; do not export MMNIST's $\lambda{=}0.8$ as a planning recipe.
Across regimes, the forcing law $L_{20}\approx a+b\eta$---checked at three $\lambda$ values on the joint grid---is more portable than any single-domain crossing event.

\section{Open problems}
\label{sec:open}

These questions sit beyond the present evidence.

\paragraph{Causal identification.}
Our mediation analysis is associational because $\lambda$ is a hyperparameter, not a randomized intervention.
A design that randomly assigns $\lambda$ (or randomly injects $\eta$) within a fixed architecture would support stronger causal language about the $\lambda\to L\to E$ path.

\paragraph{Certified geometry.}
$L_{20,\mathrm{q95}}$ is a finite-sample chord statistic.
Connecting it to global Lipschitz certificates~\citep{newhouse2026lipschitz,miller2018stable} without destroying video prediction quality remains open.

\paragraph{Planner coupling.}
Asadi's sweet spot appears under co-trained model-based control; our negative MPC result uses a frozen random-shooting planner.
A systematic study that varies both $\lambda$ (or $k$) and planner coupling would clarify when geometric regularization helps returns.

\paragraph{Richer domains.}
dm\_control cartpole-swingup is reported as a recognized-control sanity check in the main text ($11/12$ cells; no population $L{<}1$).
Larger natural-video corpora and stochastic world models (RSSM-style) remain open tests of whether the $\eta_{\mathrm{eff}}$ placement rule survives beyond Gymnasium CartPole/Pendulum and KTH.

\paragraph{Joint $(\lambda,\eta)$ surface.}
We report a dense primary curve at $\lambda{=}0.8$ and complete joint slices at $\lambda\in\{0.4,1.2\}$ ($5\eta{\times}3$ seeds each).
Both slices track approximately linear $L_{20}(\eta)$ with slopes near the primary law ($\approx 1.7$--$2.0$), so a coarse $\lambda$-dependence of the forcing slope is not evident on this grid.
A denser surface over continuous $\lambda$ (and a formal test of slope$\times\lambda$ interaction) remains open.

\paragraph{Architecture sensitivity.}
All main results use a GRU residual transition.
Whether the MMNIST $L{<}1$ threshold and the $\eta$-law slope survive under Transformer, RSSM, or JEPA-style predictors is unknown.
Architectural baselines in the defensive section (antisymmetric / spectral) are narrow substitutes, not a full architecture sweep.

\section*{Author contributions}
Both authors contributed to experimental design, analysis, and writing.
K.B.\ led paper structure and theory exposition; A.J.\ led training pipelines, analysis scripts, and artifact packaging.

\section*{Reproducibility Statement}
Headline statistics come from locked CSVs under \texttt{results/master/}.
The package validator \texttt{reproduce\_paper.py} checks metric JSONs and figures without retraining.
Checkpoints are omitted from the preprint bundle for size.
Core locks include \texttt{NUMBER\_LOCK.json} and the stage CSVs.
Training code, analysis scripts, and Kaggle kernels will be released publicly with this preprint
(the arXiv source upload contains the \LaTeX{} paper and figures only).

\section*{Ethics Statement}
Experiments use synthetic video and standard Gymnasium simulators.
No human subjects or private data are involved.
$L_{20}$ must not be read as a certified stability certificate for deployment.

\bibliography{references}
\bibliographystyle{plainnat}

\appendix
\section{Proof of Theorem~\ref{thm:recurrence}}
\label{app:recurrence}

We prove the bound by induction on $h$.
Let $z_t$ denote a reference trajectory and $\hat{z}_t$ the learned rollout, with $e_t=\|\hat{z}_t-z_t\|$.
Assume $\|\hat{z}_{t+1}-f(\hat{z}_t)\|\le\epsilon$ and $\|f(x)-f(y)\|\le L\|x-y\|$ for all $x,y$ on the rollout manifold.

\paragraph{Base case $h=1$.}
\begin{align}
e_1
&=\|\hat{z}_1-z_1\|
=\|\hat{z}_1-f(z_0)+f(z_0)-z_1\| \nonumber\\
&\le \|\hat{z}_1-f(\hat{z}_0)\| + \|f(\hat{z}_0)-f(z_0)\|
\le \epsilon + L e_0. \label{eq:base}
\end{align}

\paragraph{Inductive step.}
Assume $e_t \le L^{t} e_0 + \epsilon\sum_{j=0}^{t-1} L^{j}$ for some $t\ge 1$.
Then the same triangle inequality as in \eqref{eq:base} gives
$e_{t+1}\le \epsilon + L e_t
= \epsilon + L^{t+1} e_0 + \epsilon\sum_{j=1}^{t} L^{j}$.
Adding $\epsilon$ to the geometric sum yields
$e_{t+1}\le L^{t+1} e_0 + \epsilon\sum_{j=0}^{t} L^{j}$.
Thus the claim holds for $t+1$.

\paragraph{Closed form.}
When $L\neq 1$, $\sum_{j=0}^{h-1}L^{j}=(1-L^{h})/(1-L)$, giving the stated bound.
When $L=1$, the sum equals $h$ and $e_h\le h\epsilon$.

\paragraph{Connection to prediction MSE.}
If $e_0=0$ and $\epsilon$ upper-bounds one-step model disturbance, the bound shows exponential amplification when $L>1$ and saturation when $L<1$, matching the use of $L_{20,\mathrm{q95}}$ as a practical monitor of $k$-step gain rather than a certified global constant.

\section{Validity of \texorpdfstring{$L_{20,\mathrm{q95}}$}{L20,q95} as a local expansion statistic}
\label{app:proxy}

\subsection{Estimand and sampling protocol}
Let $f:\mathcal{Z}\to\mathcal{Z}$ be the (possibly action-conditioned) one-step latent transition composed into $f^{(k)}$.
For a validation distribution $\mu$ on $\mathcal{Z}$, draw pairs $(Z_i,Z_j)$ i.i.d.\ from $\mu\times\mu$ and define
\[
R^{(k)}_{ij}=\frac{\|f^{(k)}(Z_i)-f^{(k)}(Z_j)\|}{\|Z_i-Z_j\|},
\qquad
Q_{0.95}(R^{(k)})=\inf\{r:\mathbb{P}(R^{(k)}_{ij}\le r)\ge 0.95\}.
\]
Our reported $L_{20,\mathrm{q95}}$ is the empirical $95$th percentile of $n_{\mathrm{pairs}}$ ratios (default $n_{\mathrm{pairs}}{=}240$, $k{=}20$, stress-split latents), implemented without gradients in \texttt{empirical\_kstep\_lipschitz}.

\subsection{Upper bound under global Lipschitzness}
\begin{lemma}[Composition bound]
\label{lem:composition}
If $f$ is $L$-Lipschitz, then $f^{(k)}$ is $L^{k}$-Lipschitz, and every ratio satisfies $R^{(k)}_{ij}\le L^{k}$.
\end{lemma}
\begin{proof}
Induct on $k$.
For $k{=}1$, $R^{(k)}_{ij}\le L$ by definition.
If $\|f^{(k)}(x)-f^{(k)}(y)\|\le L^{k}\|x-y\|$, then
$\|f^{(k+1)}(x)-f^{(k+1)}(y)\|=\|f(f^{(k)}(x))-f(f^{(k)}(y))\|\le L\cdot L^{k}\|x-y\|$.
Thus $R^{(k)}_{ij}\le L^{k}$ for all pairs, and $L_{20,\mathrm{q95}}\le L^{20}$ always.
\end{proof}
Lemma~\ref{lem:composition} clarifies what $L_{20,\mathrm{q95}}<1$ \emph{does not} imply: it is a statement about \emph{empirical $20$-step pair ratios}, not that the per-step Lipschitz constant $L$ is below $1$.
It \emph{does} imply that on $95\%$ of sampled pairs the $20$-step map is non-expansive in the finite-difference sense, a strictly empirical notion tied to the pair distribution.

\subsection{Local Jacobian interpretation}
When $f$ is $C^1$, define $J_f(z)=\partial f(z)/\partial z$.
For a fixed unit direction $v$ and small $\delta>0$, Taylor expansion gives
\[
\frac{\|f^{(k)}(z+\delta v)-f^{(k)}(z)\|}{\delta}
\longrightarrow
\bigl\|(J_f(z)^{k} v)\bigr\|
\quad\text{as }\delta\to 0,
\]
provided rollout actions (if any) are matched across the pair.
Thus, for small $\|Z_i-Z_j\|$, ratios track directional $k$-step gain of the composed Jacobian along the chord $Z_j-Z_i$.
The $95$th percentile over \emph{random chords} estimates a high quantile of local directional expansion on the validation manifold, not $\sup_z \rho(J_f(z)^{k})$.

\subsection{One-step spectral probe}
For a single latent $z_0$, power iteration on $J_f(z_0)$ estimates the dominant singular direction of the \emph{one-step} Jacobian.
Let $\rho_1(z_0)=\|J_f(z_0)\|_2$ (spectral norm).
Under uniform smoothness, $\rho_1(z_0)^{k}$ upper-bounds $\|(J_f(z_0)^{k}v)\|/\|v\|$ for all $v$.
Hence a drop in mean $\rho_1$ across $z_0$ when $\lambda$ increases is a \emph{necessary directional indicator} that some local gains have shrunk, while $L_{20,\mathrm{q95}}$ aggregates $k$-step finite differences over \emph{pairs} and can be more conservative or liberal depending on chord orientation and action sampling.

\subsection{Finite-sample guarantee}
Let $\hat{Q}_{0.95}$ be the sample $95$th percentile from $n$ i.i.d.\ ratios.
Standard order-statistics theory gives $\sqrt{n}(\hat{Q}_{0.95}-Q_{0.95}(R^{(k)}))\to\mathcal{N}(0,\sigma^2)$ under smooth distributions; in practice $n{=}240$ yields stable estimates across seeds but should be read with bootstrap bands when $n_{\mathrm{seeds}}$ is small.

\subsection{Coverage conditions (when the proxy is informative)}
$L_{20,\mathrm{q95}}$ tracks local Lipschitz behavior \emph{only if}:
\begin{enumerate}
\item \textbf{Support.} Validation latents cover the manifold used at test rollouts (we use a stress split).
\item \textbf{Chord scale.} Pairs are not dominated by near-duplicate points (ratios numerically unstable) nor by extremely long chords that leave the linearization regime; our implementation clamps $\|z_i-z_j\|\ge 10^{-8}$ and samples uniformly over the pool.
\item \textbf{Action alignment.} For action-conditioned rollouts, both trajectories must see the same action sequence (random, zero, or matched protocols in Appendix~\ref{app:pendulum_mech}); otherwise ratios mix transition geometry with action noise.
\item \textbf{Differentiability.} Jacobian interpretation requires $C^1$ transitions; ReLU/tanh GRU maps are piecewise smooth, so power-iteration probes are local rather than global certificates.
\end{enumerate}

\subsection{Empirical alignment with Jacobian probe}
Short (six-epoch) MMNIST runs at $\lambda\in\{0,0.8\}$ give:
$L_{20,\mathrm{q95}}=1.295\to1.069$ and mean one-step spectral radius $1.285\to1.018$ (\texttt{jacobian\_proxy.json}).
Both statistics move in the same direction when $\lambda$ rises, as predicted when training shrinks local Jacobian gain and $k$-step chord ratios inherit that shrinkage along most sampled directions.
This is evidence that $L_{20,\mathrm{q95}}$ \emph{tracks} local stability training signals; it is not a proof that $L_{20,\mathrm{q95}}<1$ certifies global contractivity.

\section{Hyperparameters}
\label{app:hyper}
\begin{table}[h]
\centering
\caption{Default hyperparameters (\texttt{src/config.py}).}
\small
\begin{tabular}{ll}
\toprule
Parameter & Value \\
\midrule
Latent dim $d$ & $64$ (scaling runs vary; control often $32$) \\
Hidden dim & $160$ (control often $128$) \\
Optimizer & AdamW, lr $2{\times}10^{-3}$, wd $10^{-6}$ \\
Grad clip & $1.0$ \\
MMNIST epochs / batch & $20$ / full train split \\
Pendulum epochs & $25$ (CPU extras: $20$) \\
CartPole epochs & $12$--$15$ \\
Consistency horizons & $\{1,3,5,10,15,20\}$ \\
Residual scale $\alpha$ & $0.5$ \\
Image size / seq len & $32{\times}32$ / $24$ (KTH seq $16$) \\
$L_{20}$ pairs / batches & $240$ / $4$--$25$ \\
Stochastic forcing $\sigma^\star$ & $0.5$ (primary $\eta$ sweep) \\
Seeds (default) & $42,43,44$ (+$45,46,47$ where noted) \\
\bottomrule
\end{tabular}
\end{table}

\paragraph{Compute notes.}
Primary MMNIST and $\eta$-law training ran on Kaggle GPU (~$11$--$12$ min/job median).
CPU fill-ins (CartPole extras, DMC swingup) use the same code paths with \texttt{REQUIRE\_CUDA=0}.
Wall times on CPU are typically $3$--$8\times$ longer per job.

\section{MMNIST critical-pair seed table}
\label{app:mmnist-seeds}
Paired $n{=}6$ at $\lambda\in\{0,0.8\}$ (seeds $\{42,\ldots,47\}$).
Values are $L_{20,\mathrm{q95}}$ and $E_{20}$ (stress horizon $20$).
\begin{center}
\small
\begin{tabular}{lcccc}
\toprule
seed & $L_{20}(\lambda{=}0)$ & $E_{20}(\lambda{=}0)$ & $L_{20}(\lambda{=}0.8)$ & $E_{20}(\lambda{=}0.8)$ \\
\midrule
42 & $2.45$ & $0.364$ & $1.03$ & $0.177$ \\
43 & $6.76$ & $0.365$ & $0.97$ & $0.177$ \\
44 & $5.85$ & $0.364$ & $1.12$ & $0.177$ \\
45 & $3.54$ & $0.365$ & $0.96$ & $0.177$ \\
46 & $3.68$ & $0.366$ & $0.96$ & $0.177$ \\
47 & $7.46$ & $0.366$ & $1.00$ & $0.176$ \\
\bottomrule
\end{tabular}
\end{center}
Four of six seeds sit below $L{=}1$ at $\lambda{=}0.8$ ($43,45,46,47$).

\section{Pendulum mechanism ablations}
\label{app:pendulum_mech}
At $\lambda{=}0.8$, $L_{20}$ under random, zero, and matched-action protocols is approximately $1.81$, $1.85$, and $1.73$ (single-checkpoint eval).

\section{Exogenous digit table}
\begin{center}
\small
\begin{tabular}{lccc}
\toprule
Condition & $\lambda$ & $L_{20}$ & $E_{20}$ \\
\midrule
control & $0.0$ & $8.09 \pm 6.44$ & $0.363$ \\
control & $1.2$ & $1.17 \pm 0.20$ & $0.177$ \\
exogenous & $0.0$ & $5.16 \pm 1.04$ & $0.359$ \\
exogenous & $1.2$ & $0.97 \pm 0.05$ & $0.171$ \\
\bottomrule
\end{tabular}
\end{center}

\section{Scaling summary}
$d{=}16$: $L_{20}{=}0.91{\pm}0.43$ at $\lambda{=}0$; $0.69{\pm}0.14$ at $\lambda{=}0.6$; $0.75{\pm}0.20$ at $\lambda{=}0.8$.
$d{=}32$: $3.47{\pm}0.70$ / $0.89{\pm}0.11$ / $1.01{\pm}0.02$ at $\lambda\in\{0,0.6,0.8\}$.
$d{=}128$: $4.96{\pm}3.05$ / $1.20{\pm}0.05$ / $1.14{\pm}0.04$.
$d{=}64$: Table~\ref{tab:mmnist}.

\section{Mediation bootstrap}
\label{app:mediation}
Run-resampled bootstrap ($B{=}2000$) on $27$ MMNIST runs.
Primary: $\hat{r}{=}0.94$, CI $[0.88,1.00]$.
Run-level sensitivity: $\hat{r}{\approx}0.80$, CI $[0.65,0.95]$.
Source: \texttt{statistical\_audit.json}.

\section{Joint \texorpdfstring{$(\lambda,\eta)$}{(lambda, eta)} boundary}
\label{app:eta-joint}
A secondary GPU arm varied $\lambda\in\{0.4,1.2\}$ with a coarser $\eta$ grid and three seeds ($30/30$ cells locked).
Figure~\ref{fig:eta_joint} overlays those slices on the primary $\lambda{=}0.8$ curve.
\begin{figure}[H]
\centering
\includegraphics[width=0.72\linewidth]{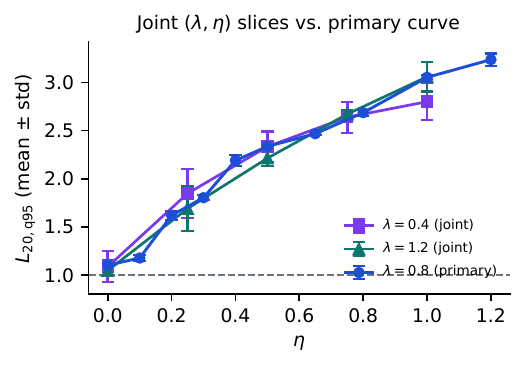}
\caption{Joint $(\lambda,\eta)$ slices at $\lambda\in\{0.4,1.2\}$ ($n{=}3$ per cell) against the primary $\lambda{=}0.8$ forcing curve ($n{=}5$). All three rise approximately linearly in $\eta$.}
\label{fig:eta_joint}
\end{figure}
Summary by $(\lambda,\eta)$:
\begin{center}
\small
\begin{tabular}{lccc}
\toprule
$\lambda$ / $\eta$ & $n$ & mean $L_{20}$ & mean $E_{20}$ \\
\midrule
$0.4$ / $0$ & $3$ & $1.09$ & $0.177$ \\
$0.4$ / $0.25$ & $3$ & $1.85$ & $0.177$ \\
$0.4$ / $0.5$ & $3$ & $2.33$ & $0.177$ \\
$0.4$ / $0.75$ & $3$ & $2.64$ & $0.177$ \\
$0.4$ / $1$ & $3$ & $2.80$ & $0.176$ \\
$1.2$ / $0$ & $3$ & $1.05$ & $0.177$ \\
$1.2$ / $0.25$ & $3$ & $1.69$ & $0.177$ \\
$1.2$ / $0.5$ & $3$ & $2.21$ & $0.177$ \\
$1.2$ / $0.75$ & $3$ & $2.67$ & $0.176$ \\
$1.2$ / $1$ & $3$ & $3.06$ & $0.177$ \\
\bottomrule
\end{tabular}
\end{center}
Both slices rise roughly linearly in $\eta$ (descriptive fits $\hat L_{20}\approx 1.30+1.69\eta$ at $\lambda{=}0.4$ and $\hat L_{20}\approx 1.13+2.00\eta$ at $\lambda{=}1.2$), consistent with the primary $\lambda{=}0.8$ law.
We do not fit a continuous surface over $\lambda$ from two slices.

\section{Statistical procedures}
\label{app:stats}
\paragraph{Paired tests.}
For MMNIST $\lambda\in\{0,0.8\}$ we pair by seed and apply a two-sided paired $t$-test and a Wilcoxon signed-rank test to $L_{20}$ and $E_{20}$.
We treat $p{<}0.05$ as the reporting threshold and quote exact $p$-values in the main text.

\paragraph{Mediation.}
Baron--Kenny path ratios use z-scored predictors on $\lambda$-averaged means for the primary point estimate, with a run-resampled bootstrap ($B{=}2000$) that redraws the $27$ runs, re-aggregates, and recomputes the ratio.
A sensitivity analysis repeats the paths on all run-level tuples without averaging.

\paragraph{$\eta$-law.}
Per-$\eta$ means use $n{=}5$ seeds with Student $t$ intervals.
The slope of $\hat L_{20}=a+b\eta$ uses a seed-block bootstrap: resample the five seed panels with replacement, refit, and take the $2.5$/$97.5$ percentiles of $b$.
Functional-form comparison (linear / threshold / power) is descriptive; we claim a law only when $R^2\ge 0.5$ and intervals exclude zero slope.

\section{s3 noisy-MNIST proxy}
\label{app:s3}
$L_{20}$ only ($E_{20}$ not logged), $n{=}3$:
$\lambda{=}0$: $1.25{\pm}0.09$; $\lambda{=}0.8$: $1.04{\pm}0.04$; $\lambda{=}1.2$: $1.04{\pm}0.01$.
\begin{figure}[h]
\centering
\includegraphics[width=0.45\linewidth]{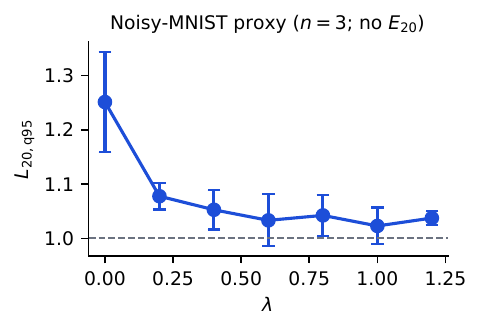}
\includegraphics[width=0.45\linewidth]{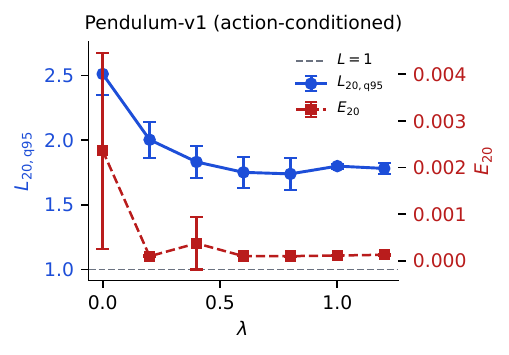}
\caption{s3 proxy ($L$ only) and Pendulum sweep.}
\label{fig:s3_appendix}
\end{figure}

\section{Jacobian probe}
\label{app:jacobian}
Short (6-epoch) MMNIST models at $\lambda\in\{0,0.8\}$; mean one-step Jacobian spectral radius via power iteration on $\partial z_{t+1}/\partial z_t$ (\texttt{scripts/jacobian\_proxy\_eval.py}).
At $\lambda{=}0$: $L_{20,\mathrm{q95}}{=}1.295$, mean spectral radius $1.285$.
At $\lambda{=}0.8$: $L_{20,\mathrm{q95}}{=}1.069$, mean spectral radius $1.018$.
Both drop with $\lambda$, consistent with Appendix~\ref{app:proxy}; neither certifies $\sup\|J_f\|<1$.
Source: \texttt{jacobian\_proxy.json}.

\section{CartPole table}
\label{app:cartpole}
\begin{center}
\small
\begin{tabular}{lcc}
\toprule
$\lambda$ & $L_{20}$ & $E_{20}$ \\
\midrule
$0.0$ ($n{=}5$) & $2.96{\pm}0.26$ & ${\sim}4{\times}10^{-4}$ \\
$0.4$ ($n{=}5$) & $3.18{\pm}0.24$ & ${\sim}1{\times}10^{-4}$ \\
$0.8$ ($n{=}5$) & $3.17{\pm}0.42$ & ${\sim}1{\times}10^{-4}$ \\
$1.2$ ($n{=}5$) & $2.99{\pm}0.20$ & ${\sim}1{\times}10^{-4}$ \\
\bottomrule
\end{tabular}
\end{center}
CPU seed bumps $\{45,46\}$ are locked at all four $\lambda$ values ($n{=}5$).

\section{WorldTest details}
MFP 24 tasks, CD 16 tasks; behavioral score $0$--$5$~\citep{warrier2025autumnbench}.
Balanced $n{=}5$ (seeds $\{42,\ldots,46\}$).
In-domain $\Delta{+}0.64$; grid OOD $\Delta{-}0.70$.

\section{Compute log}
\begin{center}
\small
\begin{tabular}{lrr}
\toprule
Stage & Jobs & Notes \\
\midrule
s1 MMNIST & $27$ & $7\lambda{\times}3$ plus critical-pair extras \\
s1 Pendulum & $21$ & $7\lambda{\times}3$ \\
s2 baselines & $6$ & antisymmetric + spectral \\
s2 exogenous & $12$ & $2{\times}2{\times}3$ \\
s3 proxy & $21$ & noisy MNIST \\
s4 scaling & $27$ & $d\in\{16,32,128\}$ \\
s5 CartPole & $12$ & $4\lambda{\times}3$ \\
s7 KTH & $12$ & $4\lambda{\times}3$ \\
s8 $\eta$-law primary & $50$ & $\lambda{=}0.8$, $10\eta{\times}5$ \\
s8 $\eta$-law joint & $30$ & $\lambda\in\{0.4,1.2\}$, $5\eta{\times}3$ \\
CPU DMC swingup & $11/12$ & cartpole-swingup; missing $\lambda{=}1.2$ seed~$44$ \\
CPU seed bump CartPole & $8/8$ & seeds $45$--$46$ $\to$ $n{=}5$ all $\lambda$ \\
CPU seed bump Pendulum & $3/8$ & $\lambda{=}0$ at $n{=}5$; rest partial \\
\midrule
Total logged & ${\sim}200{+}$ & ${\sim}50$ GPU-h + CPU marathon \\
\bottomrule
\end{tabular}
\end{center}

\section{Control-domain \texorpdfstring{$L_{20}$}{L20} bootstrap intervals}
\label{app:control-boot}
Seed-block bootstrap means for $L_{20,\mathrm{q95}}$ ($B{=}2000$) on merged CSVs (CartPole $n{=}5$ all $\lambda$; Pendulum $n{=}5$ at $\lambda{=}0$, $n{=}4$ at $\lambda{=}0.4$, $n{=}3$ elsewhere; KTH $n{=}3$).
\begin{center}
\small
\begin{tabular}{lccc}
\toprule
$\lambda$ & Pendulum & CartPole & KTH \\
\midrule
$0$ & $2.51$ [$2.39,2.64]$ & $2.96$ [$2.78,3.18]$ & $3.32$ [$1.73,6.28]$ \\
$0.4$ & $1.83$ [$1.73,1.92]$ & $3.18$ [$3.01,3.37]$ & $1.97$ [$1.13,3.57]$ \\
$0.8$ & $1.74$ [$1.63,1.87]$ & $3.17$ [$2.88,3.51]$ & $1.41$ [$1.12,1.83]$ \\
$1.2$ & $1.78$ [$1.74,1.82]$ & $2.99$ [$2.82,3.14]$ & $1.54$ [$1.15,2.30]$ \\
\bottomrule
\end{tabular}
\end{center}
All intervals lie above $1$; none of these domains show a population $L{<}1$ event on the locked grids.

\section{Reporting checklist for geometry diagnostics}
\label{app:checklist}
We recommend the following minimum report when claiming a contraction or expansion event for a latent world model:
\begin{enumerate}
\item \emph{Proxy definition.} Exact formula for the expansion statistic (here $L_{20,\mathrm{q95}}$), pair count, action protocol, and validation pool.
\item \emph{Seed table.} Per-seed values at the critical operating point, not only means.
\item \emph{Paired tests.} When claiming a $\lambda$ threshold, report paired tests on the same seeds.
\item \emph{Domain boundary.} At least one action-conditioned or natural-video negative, so $L{<}1$ is not over-exported.
\item \emph{Planning readout (optional but clarifying).} A frozen-checkpoint planner score, even if negative, to separate geometry from control utility.
\item \emph{Compute and incompleteness.} GPU hours and which appendix tables remain partial (DMC $11/12$; Pendulum seed extras).
\end{enumerate}
Our main text and appendices are written to satisfy (1)--(6) for the MMNIST critical pair and the $\eta$-law primary arm.

\section{Extended discussion of the \texorpdfstring{$L{=}1$}{L=1} reference line}
\label{app:L1}
The choice of $L{=}1$ as a reference line is motivated by Theorem~\ref{thm:recurrence}: under a global Lipschitz assumption, $L{<}1$ saturates disturbance accumulation while $L{>}1$ permits exponential growth.
In practice we never observe a global Lipschitz certificate; we observe a high quantile of finite-difference $k$-step ratios.
Crossing $L_{20,\mathrm{q95}}{<}1$ therefore means that \emph{on the measured validation pool}, $95\%$ of sampled chords are non-expansive under the $20$-step map.
It does \emph{not} mean:
\begin{itemize}
\item that every direction of the Jacobian spectrum is contractive;
\item that planning returns improve;
\item that the same $\lambda$ will produce $L{<}1$ on another domain;
\item that the model is safe for open-loop deployment.
\end{itemize}
We keep the reference line because it is interpretable, comparable across seeds, and aligned with the recurrence bound's qualitative split.
Alternative references (e.g., mean spectral radius $=1$, or a planner-calibrated threshold) are left as open measurement choices.

\section{Reproducibility map}
\label{app:repro}
Primary artifacts live under \texttt{results/master}.
Headline numbers are locked in \texttt{NUMBER\_LOCK.json}
and checked by \texttt{reproduce\_paper.py}.
Training entry points live under \texttt{src/};
analysis and packaging scripts live under \texttt{scripts/}.
CPU fill-in scripts write resume-safe CSVs.
The arXiv source bundle contains the compiled paper assets;
code and CSV locks are released separately with the public repository.

\end{document}